\documentclass[conference]{IEEEtran}
\IEEEoverridecommandlockouts

\usepackage{cite}
\usepackage{amsmath,amssymb,amsfonts}
\usepackage{graphicx}
\usepackage{textcomp}
\usepackage{xcolor}
\usepackage{booktabs}
\usepackage{array}
\usepackage{url}

\newcolumntype{L}[1]{>{\raggedright\arraybackslash}p{#1}}
\newtheorem{proposition}{Proposition}
\newtheorem{definition}{Definition}
\newcommand{\E}{\mathbb{E}}

\begin{document}

\title{Beyond Reward Hacking: Proxy Divergence Across\\
Four Layers of a Staged Humanoid Learning Pipeline}

\author{
\IEEEauthorblockN{Arunabh Bora}
\IEEEauthorblockA{\textit{London, United Kingdom}\\
arunabhbora.ab@gmail.com}
}

\maketitle

%=======================================================================
\begin{abstract}
A reinforcement-learning (RL) pipeline for a legged robot is assembled from proxies. A
reward stands in for intended behaviour, a curriculum gate stands in for competence, an
evaluation statistic stands in for robustness, and a reference motion stands in for an
achievable skill. The traditional view treats only the first of these as optimised
against, and so locates specification failure (reward hacking) in the reward alone. I
argue that all four are proxies in the same formal sense, that each has a characteristic
divergence mechanism, and that each admits a reformulation that closes it. For every
layer I state the traditional formulation, derive the condition under which it diverges
from its target, and give the alternative: first-order (L1) costs where quadratic
kernels are flat, peak and outcome statistics where curriculum gates average, gate
reachability and information checks, deterministic and phase-desynchronised evaluation,
curriculum state treated as part of the model, feasibility-first reference design with
residual feed-forward, and function-preserving input widening that lets one policy grow
instead of being retrained. The arguments are illustrated by measurements from one
continuous lineage of a PPO policy for a simulated 1.91\,m humanoid, grown over four
stages and 13{,}500 iterations on a single laptop GPU. Among them, a curriculum gate built
on averaged error advanced at its rate limit on every check while the skill it gated was
absent, and a batched push test whose synchronised resets aliased the gait phase ranked a
0.5\,m/s push as more dangerous than a 2.0\,m/s one.
\end{abstract}


%=======================================================================
\section{Introduction}
%=======================================================================

\IEEEPARstart{G}{PU}-parallel simulation has made it routine to train a locomotion policy
in hours on one machine \cite{rudin2021,mittal2025}. The resulting pipeline has a
standard shape: a weighted sum of shaped reward terms, a performance-gated curriculum over
terrain and commands, batched evaluation in the simulator, and, for skills beyond
locomotion, a tracking reward towards a reference motion \cite{peng2018}. Each of these
components replaces something physical with something computable.

The literature on specification gaming \cite{skalse2022,pan2022,gao2023,laidlaw2025} studies
the gap between a proxy and its target, but in robot learning it is applied almost
exclusively to the reward. The other components are treated as plumbing: a curriculum
threshold is a tuning constant, an evaluation script is assumed to measure what it
reports, and a reference motion is assumed to be achievable. My thesis is that this
traditional view is incomplete. A curriculum gate is optimised against just as a reward
is, because the policy that passes it is the one the gate selects. An evaluation
statistic is a proxy whose bias can be created by the harness itself, with no
optimisation at all. A reference motion is a proxy for a skill, and it can be wrong in a
way that no reward design can repair.

\subsection{Contributions}
\begin{enumerate}
\item \textbf{A four-layer formulation of proxy divergence} (reward, curriculum,
      measurement, reference), with, for each layer, the traditional formulation, the
      condition under which it diverges, and a reformulation (Sections
      \ref{sec:reward}--\ref{sec:reference}, summarised in Table~\ref{tab:summary}).
\item \textbf{Curriculum gates as estimators}: a dilution bound showing that gates on
      error averaged over joints and time cannot reject large localised failures, a
      reachability condition, an information argument for counting gate refusals, and the
      observation that curriculum state belongs to the model, not the environment
      (Section~\ref{sec:curriculum}).
\item \textbf{Two systematic evaluation biases} that do not shrink with sample size:
      outcomes manufactured by exploration noise, and phase aliasing from synchronised
      resets, with corrected estimators (Section~\ref{sec:measurement}).
\item \textbf{Feasibility-first reference design}: why tracking an infeasible reference
      converges to a qualitatively different behaviour regardless of reward shaping, and
      a static test, a correction and a residual action that avoid it
      (Section~\ref{sec:reference}).
\item \textbf{Function-preserving policy growth} through input widening and split
      authority, as an alternative to retraining from scratch whenever the observation
      changes (Section~\ref{sec:growth}).
\end{enumerate}

%=======================================================================
\section{Related Work}
%=======================================================================

\textbf{RL for legged and humanoid robots.} Deep RL now produces deployed locomotion and
whole-body controllers \cite{tang2025}, including zero-shot humanoid walking
\cite{radosavovic2024}, agile whole-body skills with sim-to-real alignment \cite{he2025},
and general motion tracking \cite{liao2025}. Most of these pipelines share the structure
analysed here, built on GPU-parallel simulators such as Isaac Lab \cite{mittal2025}.

\textbf{Reward design and specification gaming.} Practical locomotion rewards are weighted
sums of heuristic terms \cite{rudin2021}, and recent work automates their design with
language models \cite{ma2024}. Skalse \emph{et al.} \cite{skalse2022} showed that
unhackable proxy rewards are very restricted; Pan \emph{et al.} \cite{pan2022} found that
more capable agents exploit misspecified rewards more, sometimes through abrupt phase
transitions; Gao \emph{et al.} \cite{gao2023} measured scaling laws for over-optimising a
proxy; and Laidlaw \emph{et al.} \cite{laidlaw2025} defined reward hacking through the
correlation between proxy and true reward under a base policy. All of this work treats the
reward as the proxy. I extend the same lens to curriculum gates, evaluation statistics and
reference motions.

\textbf{Curricula.} In locomotion, terrain and command levels are raised when agents
perform well \cite{rudin2021}, and adaptive command curricula are central to fast legged
locomotion \cite{margolis2024}. The gate that decides promotion is treated as a tuning
constant; I analyse it as an estimator.

\textbf{Evaluation methodology.} Agarwal \emph{et al.} \cite{agarwal2021} and Patterson
\emph{et al.} \cite{patterson2024} showed that deep-RL conclusions are fragile to seeds,
aggregate statistics and experimental design. Those are largely variance effects; the
biases studied here are systematic and do not shrink with sample size.

\textbf{Imitation and residual control.} Tracking rewards towards reference motions
\cite{peng2018} and adversarial motion priors \cite{peng2021} remain the basis of humanoid
skill learning \cite{he2025,liao2025}. References retargeted from human motion capture
approximately inherit the human's dynamic feasibility; hand-authored references do not.
Residual formulations learn a correction on top of a fixed controller
\cite{johannink2019}, and ASAP applies a learned delta-action model on top of a pretrained
policy \cite{he2025}.

\textbf{Growing and continually training networks.} Net2Net \cite{chen2016} introduced
function-preserving expansion. Plasticity injection \cite{nikishin2023} adds a new
network whose initial output is exactly zero, so predictions are unchanged at the moment
of injection; Neuroplastic Expansion \cite{liu2025} grows the network during training.
Loss of plasticity under continual training is now well documented \cite{dohare2024},
including for on-policy methods such as PPO \cite{juliani2024}, which is the regime of a
policy fine-tuned across many stages.

%=======================================================================
\section{Setting}
\label{sec:setting}
%=======================================================================

The robot is a simulated 1.91\,m, 69.35\,kg humanoid built directly from its CAD
assembly. The policy controls 11 joints (hip pitch, hip roll, knee, ankle
pitch and ankle roll on each leg, and a waist yaw); 10 arm and neck joints are on position
drives outside the policy. It is a Gaussian MLP (three hidden layers of 128 ELU units,
separate critic) trained with PPO \cite{schulman2017} and GAE \cite{schulman2015}
($\gamma=0.99$, $\lambda=0.95$, clip 0.2, adaptive learning rate at target KL 0.01) on
2048 environments $\times$ 24 steps per iteration. Actions are joint offsets,
$q^{\ast}=q_0+0.5\,a$, tracked by implicit PD drives at 50\,Hz over 200\,Hz physics.
Observations include base velocities, projected gravity, commands, joint states, the
previous action, and a 187-ray terrain height scan. Terrain is procedurally generated
under a promotion curriculum, and random velocity pushes are applied. All training used a
single RTX~5070~Ti laptop GPU.

I write $r_t=\sum_i w_i r_i(s_t,a_t)$ for the reward, $k$ for a curriculum level that
indexes a task distribution $d_k$, and $J(\theta;d)=\E_{\xi\sim d}\,\E_{\pi_\theta}\left[\sum_t
\gamma^t r_t\right]$ for the objective.

\begin{definition}
A \emph{proxy} is a computable quantity $m$ used to steer or judge learning in place of a
physical target $T$. A \emph{proxy divergence} is a policy, or a population of policies,
for which $m$ reports success while $T$ is absent.
\end{definition}

%=======================================================================
\section{Reward Layer}
\label{sec:reward}
%=======================================================================

\subsection{Second-order kernels permit first-order drift}

\textbf{Traditional.} Command tracking uses a Gaussian kernel,
$r_{\mathrm{trk}}=\exp(-\|\mathbf{c}-\mathbf{v}\|^2/\sigma^2)$, chosen because it is
bounded and smooth \cite{rudin2021}.

\textbf{Divergence.} Consider a zero command and a scalar drift speed $v\ge0$. Let $g(v)$
collect every other effect of drifting on return, for instance that stepping slightly
forward is easier to balance than standing still.

\begin{proposition}
If $g$ is differentiable at 0 with $g'(0)>0$, then $f(v)=\exp(-v^2/\sigma^2)+g(v)$ is
not maximised at $v=0$. If instead the tracking term is augmented with an L1 cost
$-\lambda|v|$, then $v=0$ is a local maximiser whenever $g'(0)<\lambda$.
\end{proposition}
\begin{IEEEproof}
The kernel has zero derivative at 0, so $f'(0)=g'(0)>0$ and $f$ increases to the right
of 0. With the L1 term, the right derivative at 0 on the domain $v\ge0$ is
$g'(0)-\lambda<0$, so 0 is a local maximiser.
\end{IEEEproof}

The kernel's loss near zero is $1-e^{-v^2/\sigma^2}\approx v^2/\sigma^2$, second order,
so any first-order incentive wins. With $\sigma=0.5$ the kernel still pays 0.85 at
0.2\,m/s of drift, and a policy trained this way crept forward on every zero command.

\textbf{Reformulation.} Pair every saturating kernel whose target is a rest state with a
first-order cost active only in that state. An L1 base-velocity cost of weight 1.5 on zero
commands reduced measured drift to 0.01--0.04\,m/s.

\subsection{Unconstrained coordinates settle on their bounds}

\textbf{Traditional.} Penalty terms are added for joints whose motion is undesirable and
omitted elsewhere, often by inheriting a reference configuration from another robot.

\textbf{Divergence.} If the reward has no term in a coordinate $q_j$, its value is set by
whatever residual, incidental dependence $g(q_j)$ the rest of the objective has. If $g$
is monotone on the feasible interval $[q_{\min},q_{\max}]$, its maximiser is a bound,
where the joint limit supplies the only opposing force. A quadratic regulariser
$-\mu(q_j-q_j^0)^2$ makes the objective strictly concave in $q_j$ and moves the
maximiser into the interior whenever $|g'|<2\mu\,(q_{\max}-q_j^0)$.

On rough terrain the waist penalty was removed, because the configuration it came from
belongs to a robot that turns with hip yaw. The waist then sat at its $+45.00^\circ$
limit for 100\,\% of steps with zero variance, while every logged statistic improved
(reward $+27.4$ against $+25.7$, falls 2.8\,\% against 5.3\,\% for the corrected policy).

\textbf{Reformulation.} Every actuated coordinate must be in the support of the reward,
either through a regulariser or a commanded target, and inside a limit barrier. In the
final configuration the waist is not penalised but \emph{commanded}: a target in
$\pm0.7$\,rad, zero half the time, is observed and rewarded by
$\exp(-(q_w-q_w^\ast)^2/0.2^2)$, which subsumes the regulariser.

%=======================================================================
\section{Curriculum Layer}
\label{sec:curriculum}
%=======================================================================

\textbf{Traditional.} A level $k$ is raised when a statistic $\bar m$ of recent
episodes crosses a threshold, $k\leftarrow k+1$ if $\bar m\ge\tau$ (or $\bar
e<\tau$ for an error), usually with a rate limit.

\subsection{Dilution}

Let a reference-tracking skill be judged by tracking error averaged over $|\mathcal J|$
joints and the $T$ steps of the motion,
$\bar e=\frac{1}{|\mathcal J|T}\sum_{j,t}|e_j(t)|$. Suppose one joint misses by $e^\ast$
during a fraction $\rho$ of the motion and every other error is at most $\epsilon$.

\begin{proposition}
$\bar e\le \rho e^\ast/|\mathcal J|+\epsilon$, so the gate $\bar e<\tau$ accepts every
such failure with $e^\ast<|\mathcal J|(\tau-\epsilon)/\rho$.
\end{proposition}

The bound is direct. With ten tracked joints, $\tau=0.25$\,rad and $\rho\approx0.3$, the
gate accepts misses up to several radians, i.e.\ it cannot reject any physically possible
failure of the principal joint. In the lineage, a gate of exactly this form accepted a
single-support leg-lift motion in which the lifting hip missed its target by
$59^\circ\approx1.03$\,rad; the logged mean error was 0.06--0.07\,rad
(Fig.~\ref{fig:curriculum}). A no-fall criterion did not help, because the degenerate
solution, keeping both feet down, never falls.

\textbf{Reformulation.} Gate on the peak, $e_\infty=\max_t\max_{j\in\mathcal
J_{\mathrm{act}}}|e_j(t)|$, over the joints the motion actually moves, averaged over
episodes; and conjoin it with an \emph{outcome predicate} measured at the contact level
(here, the lifted foot carries less than 1\,N for a minimum fraction of the swing and
reaches a minimum height relative to what the reference produces).

\begin{figure}[t]
\centering
\includegraphics[width=\columnwidth]{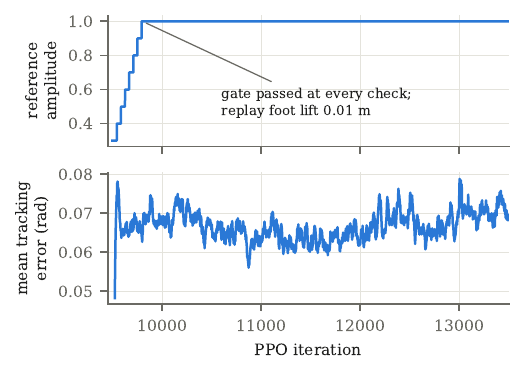}
\caption{A curriculum gated on averaged tracking error. Top: the reference amplitude
advanced at its rate limit at every check. Bottom: the gated statistic stayed far below
its 0.25\,rad threshold, although the principal joint's peak error was
$\approx$1\,rad.}
\label{fig:curriculum}
\end{figure}

\subsection{Information and refusal}

Let $G\in\{0,1\}$ be the gate's decision and $C$ the competence it should detect. If the
gate passes at every check, $H(G)=0$ and therefore $I(G;C)=0$: the gate transmits no
information about competence. A gate that advances at exactly its rate limit is
observationally equivalent to no gate. In Fig.~\ref{fig:curriculum} the seven raises
came at intervals of 41 to 44 iterations against a rate limit of one episode
($\approx$41.7 iterations), with zero refusals.

\textbf{Reformulation.} Log the refusal count as a first-class diagnostic, and treat a
curriculum with no refusals as unvalidated.

\subsection{Reachability}

A threshold is useful only if some policy can meet it:
$\sup_\pi\E_\pi[m]\ge\tau$. Executing the reference open-loop is itself a policy
$\pi_{\mathrm{ref}}$, so $\E_{\pi_{\mathrm{ref}}}[m]\ge\tau$ is a sufficient condition
that costs one rollout. When it fails, the gate may be unreachable. In the lineage, an
air-time requirement of 50\,\% of the swing exceeded the $\approx$30\,\% that the reference
itself achieves at half amplitude ($\approx$59\,\% at full), and a starting amplitude of
0.3 produced a lift too small to clear the ground once the support leg tilted.

\textbf{Reformulation.} Set $\tau\le\E_{\pi_{\mathrm{ref}}}[m]$ at every curriculum level
before training.

\subsection{Curriculum state is model state}

\textbf{Traditional.} A checkpoint stores network weights, optimiser moments and an
iteration counter. Curriculum levels live in the environment.

\textbf{Divergence.} Fine-tuning from a checkpoint optimises $J(\theta;d_{k_0})$ rather
than $J(\theta;d_{k})$: the objective silently changes. Skills trained only at high levels
receive no gradient to maintain them, and statistics compared across stages are
expectations under different distributions. In the lineage, a speed curriculum that
had reached 1.7\,m/s restarted at 1.0\,m/s in the next stage and was not raised again for
2000 iterations; terrain level fell from 5.6 to 3.5 at the same boundary.

\textbf{Reformulation.} Checkpoint the tuple $(\theta,\text{optimiser},k)$, and compare
stages only under a fixed evaluation distribution $d_{\mathrm{eval}}$.

%=======================================================================
\section{Measurement Layer}
\label{sec:measurement}
%=======================================================================

\subsection{Outcomes manufactured by exploration noise}

\textbf{Traditional.} Outcome statistics are logged from training rollouts of the
stochastic policy $a=\mu_\theta(s)+\sigma\varepsilon$.

\textbf{Divergence.} For a thresholded outcome $m=\mathbf 1[h(s)>h_0]$,
$\E_\varepsilon[m]$ can be positive when the deterministic rollout has $h<h_0$
everywhere, because noise alone reaches the threshold. The estimate converges, with
more samples, to the noise's contribution, not the policy's. In the lineage a foot-height
metric read 61--69\,\% of target during training while the deterministic policy produced
none: exploration noise of $\approx$40$^\circ$ on the ankle targets (against
$\approx$6$^\circ$ at the hips) tipped a foot carrying $\approx$360\,N onto its toe,
raising its centre of mass.

\textbf{Reformulation.} Report outcomes for $\mu_\theta$, and define outcomes by
predicates that noise cannot satisfy cheaply: here, height counts only while the foot's
contact force is below 1\,N.

\subsection{Phase aliasing in batched evaluation}

\textbf{Traditional.} A disturbance test resets $N$ parallel environments together,
waits a fixed settling time $t_0$, applies the disturbance, and reports the failure
fraction $\hat p=\frac1N\sum_i f_i$.

\textbf{Divergence.} A gait is approximately periodic with period $P$, so failure
depends on the phase $\phi$ at which the disturbance arrives, $p(m\,|\,\phi)$. With
synchronised resets, every robot has $\phi_i\approx t_0 \bmod P$, and
\begin{equation}
\hat p \xrightarrow{N\to\infty} p(m\,|\,\phi_0)\neq \bar p(m)=\frac1P\int_0^P p(m\,|\,\phi)\,d\phi .
\end{equation}
The bias is independent of $N$: more robots make a wrong number more precise. In the
lineage this produced a fall rate of 88\,\% for a 0.5\,m/s push from one side against
50\,\% for 2.0\,m/s from the same side, together with a large and entirely spurious
left--right asymmetry.

\textbf{Reformulation.} Draw each robot's disturbance time independently,
$t_i\sim\mathcal U[t_0,t_0+W]$. The induced phase density is $n/W$ or $(n+1)/W$ with
$n=\lfloor W/P\rfloor$, so it differs from the uniform $1/P$ by at most $1/W$, and is
exactly uniform when $W$ is a multiple of $P$. With $W=1$\,s, longer than one step cycle,
the magnitude series became monotonic (Table~\ref{tab:push}) and the asymmetry
disappeared; lateral and sagittal pushes then had near-equal fall rates (30\,\% against
26\,\% while walking).

\begin{table}[t]
\caption{Push Recovery After Phase Randomisation (\% of 32 Robots Falling, Pooled Over
Four Directions; Training Pushes Were $\pm$0.5\,m/s)}
\label{tab:push}
\centering
\footnotesize
\begin{tabular}{@{}lrr@{}}
\toprule
\textbf{Push (m/s)} & \textbf{Walking 0.6\,m/s} & \textbf{Standing} \\
\midrule
0.5 & 3  & 10 \\
1.0 & 27 & 38 \\
1.5 & 35 & 44 \\
2.0 & 47 & 62 \\
\bottomrule
\end{tabular}
\end{table}

\subsection{Quantifier errors in aggregation}

A claim of the form ``the policy survives magnitude $m$ in every direction'' is a
universal quantifier and must be computed with a minimum over directions. A harness that
computed the largest magnitude with zero falls in \emph{any} direction (an existential)
reported a clean result while another direction had a 38\,\% fall rate. The statistic's
quantifier must match the sentence it is used to support.

%=======================================================================
\section{Reference Layer}
\label{sec:reference}
%=======================================================================

\textbf{Traditional.} A skill is specified by a reference trajectory $q_{\mathrm{ref}}(t)$
and learned by maximising $\exp(-\|q-q_{\mathrm{ref}}\|^2_M/\sigma^2)$ alongside a large
failure penalty, with reward terms added when the result is unsatisfactory
\cite{peng2018}.

\textbf{Divergence.} Let $\mathcal F$ be the set of trajectories the robot can execute
without falling. If $q_{\mathrm{ref}}\notin\mathcal F$, a policy cannot track it, and
the large failure penalty drives the optimum to a trajectory in $\mathcal F$ that is close
to $q_{\mathrm{ref}}$ in the reward's metric $M$. Nothing requires that projection to
share the reference's qualitative structure (its contact sequence). Additional shaping
terms change $M$ and therefore which feasible trajectory is chosen, but they cannot make
the infeasible one feasible.

A static necessary condition for a single-support phase is that the centre of mass
projects into the support polygon, $p^{xy}_{\mathrm{com}}(q_{\mathrm{ref}}(t))\in\mathcal
S(t)$. In the lineage, hand-authored single-support leg-lift references satisfied every
joint limit, yet in the lifted pose the centre of mass lay 0.21\,m beside the support
foot. Seven reward, initialisation and exploration changes were tried against that
reference; each moved its own shaped term, and none produced single support in the
deterministic policy.

\textbf{Reformulation, part 1: feasibility first.} Before any reward design, check the
static condition at every keyframe and repair violations kinematically. Here a support-leg
roll $\theta$, with ankle roll $-\theta$ to keep the sole flat, moves the centre of mass
at a measured $\kappa\approx0.0095$\,m per degree, so the required shift is
$\theta\approx d/\kappa=0.21/0.0095\approx22^\circ$; at $20^\circ$, within the
$27^\circ$ limit, the residual offset is 0.02\,m. Then execute the repaired reference
open-loop as a dynamic check: if it reproduces the intended contact sequence for part of
the motion, the reference is a sound target.

\textbf{Reformulation, part 2: residual feed-forward.} Rather than asking the policy to
rediscover the reference from reward, feed it forward:
\begin{equation}
q^\ast=q_0+0.5\,a+r(t),
\end{equation}
with $r(t)=0$ when no reference is active. Then $a=0$ reproduces the (feasible) reference,
exploration is centred on it, and the learning problem changes from discovery to
regulation, as in residual RL \cite{johannink2019,he2025}. The action space is
unchanged, so existing checkpoints load without modification. Measured on the repaired
leg-lift reference, open-loop feed-forward alone held the foot clear of the ground for
the whole motion, whereas a policy required to discover the support-leg shift from
reward held that joint at $-1$ to $-3^\circ$ against a $-20^\circ$ target.

%=======================================================================
\section{Policy Growth}
\label{sec:growth}
%=======================================================================

\subsection{Function-preserving input widening}

\textbf{Traditional.} When a new capability requires new observations, the input
dimension changes, the checkpoint no longer loads, and training restarts from scratch.

\textbf{Reformulation.} Append the new inputs $\mathbf x_{\mathrm{new}}\in\mathbb R^m$
after the old ones and pad the first-layer weights of actor and critic with zeros,
\begin{equation}
W'=\begin{bmatrix}W&\mathbf 0\end{bmatrix},\qquad
W'\begin{bmatrix}\mathbf x\\\mathbf x_{\mathrm{new}}\end{bmatrix}=W\mathbf x .
\end{equation}
Actions and values are preserved exactly, so advantage estimates are not disturbed at the
handover. This is the input-side counterpart of plasticity injection \cite{nikishin2023}
and Net2Net \cite{chen2016}, which preserve a network's function while adding capacity;
here the added capacity is a new observation channel. Unlike a fully zero-initialised layer, the padded block still receives gradient,
$\partial L/\partial W_{\mathrm{new}}=\boldsymbol\delta\,\mathbf x_{\mathrm{new}}^{\top}$,
which is non-zero whenever the new inputs vary, so they are learnable from the first
update. Adam's moment estimates have the old shape and are discarded; the adaptive-KL
learning rate absorbs the resulting transient. The exploration-noise parameterisation
must also be read from the checkpoint, since $\log\sigma$ and $\sigma$ are stored under
different keys. Widening 253 inputs to 274 reproduced the parent's actions to within
$10^{-5}$.

\subsection{Split authority with observed exogenous motion}

\textbf{Traditional.} Joints outside the policy are held fixed in training, so the
training distribution of their state is a point mass at rest. Any later motion of those
joints is a distribution shift. In the lineage, a policy trained with fixed arms fell
when an arm was held out sideways: tilt grew from 0.24 to 0.97 over two seconds, because
the sustained lateral moment loads the hip-roll joint, already the most heavily loaded
in walking.

\textbf{Reformulation.} Drive the non-policy joints during training with an exogenous
process (rest poses, keyframe sequences at randomised speed, random poses within soft
limits, all smoothly ramped), and include their positions and velocities, but not the
command generating them, in the observation. The policy is then a function of the
exogenous joint state alone, so at deployment any source of motion, whether a script, a
teleoperator or another controller, presents identically. Joints the policy \emph{does}
own are steered by commanded targets rather than external scripts, as for the waist in
Section~\ref{sec:reward}.

\subsection{Empirical lineage}

Table~\ref{tab:lineage} and Fig.~\ref{fig:lineage} summarise one parameter lineage
grown this way. S1 trained from scratch with standing, observed upper-body motion and a
torso target; S2 added a forward-speed curriculum to 2\,m/s; S3 added fast upper-body
keyframe sequences that drive the torso target; S4 widened the input with a 21-dimensional
lower-limb reference (current and 0.2\,s ahead, plus an activity flag). In replay, the
S1 policy drifts 0.01--0.04\,m/s on a zero command and holds a $+34^\circ$ torso command
at 28--36$^\circ$. S3 reduced peak tilt during the fastest arm sequence from 0.34 to 0.21
and base drift from 0.37\,m to 0.14\,m. The four stages took $\approx$14.2 wall-clock
hours and $6.6\times10^{8}$ samples.

\begin{table}[t]
\caption{One Parameter Lineage. Stage-End Values Are Means Over the Final 100
Iterations of Training Telemetry; per Section~\ref{sec:curriculum}, They Are Not
Measured on a Common Distribution.}
\label{tab:lineage}
\centering
\footnotesize
\setlength{\tabcolsep}{3pt}
\begin{tabular}{@{}lL{2.3cm}rrrrr@{}}
\toprule
\textbf{Stage} & \textbf{Adds} & \textbf{Iter.} & \textbf{Obs} & \textbf{Vel.\,trk} & \textbf{Fall} & \textbf{Waist} \\
 & & & & & \textbf{(\%)} & \textbf{err (rad)} \\
\midrule
S1 & standing, observed upper body, torso target (scratch) & 0--4499 & 253 & 0.760 & 1.27 & 0.057 \\
S2 & speed curriculum to 2\,m/s & --7498 & 253 & 0.717 & 1.73 & 0.067 \\
S3 & fast upper-body sequences & --9497 & 253 & 0.758 & 1.15 & 0.065 \\
S4 & lower-limb reference (widened) & --13496 & 274 & 0.733 & 1.03 & 0.067 \\
\bottomrule
\end{tabular}
\end{table}

\begin{figure}[t]
\centering
\includegraphics[width=\columnwidth]{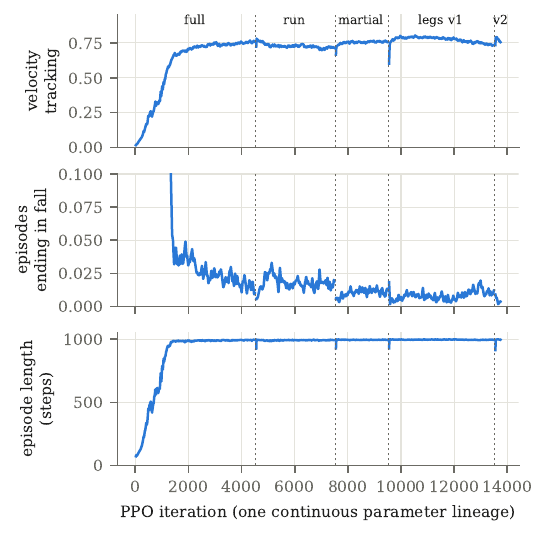}
\caption{Training telemetry across the lineage; dashed lines mark stage boundaries. Each
fine-tuning stage begins with randomised episode lengths, so its first 30 iterations are
omitted. The fall panel is clipped; S1 starts from scratch at 94\,\% falls.}
\label{fig:lineage}
\end{figure}

%=======================================================================
\section{Traditional Versus Proposed: Summary}
\label{sec:protocol}
%=======================================================================

Table~\ref{tab:summary} collects the reformulations. Read as a checklist, it is a
proxy audit: for each quantity that steers or judges learning, ask which column of the
table it is in.

\begin{table*}[t]
\caption{Traditional Formulations, Their Divergence Mechanisms, and the Proposed
Alternatives}
\label{tab:summary}
\centering
\footnotesize
\setlength{\tabcolsep}{4pt}
\begin{tabular}{@{}lL{3.9cm}L{5.0cm}L{6.2cm}@{}}
\toprule
\textbf{Layer} & \textbf{Traditional} & \textbf{Why it diverges} & \textbf{Proposed} \\
\midrule
Reward & Gaussian tracking kernel & Loss is second order at the rest state; any first-order incentive causes drift & Add a first-order (L1) cost active in the rest state \\
Reward & Penalise only undesirable joints & An unrewarded coordinate with monotone residual objective settles on its bound & Every actuated coordinate in reward support (regulariser or commanded target) and behind a limit barrier \\
\midrule
Curriculum & Gate on error averaged over joints and time & Dilution: accepts misses up to $|\mathcal J|(\tau-\epsilon)/\rho$ & Peak error over actuated joints, conjoined with a contact-level outcome predicate \\
Curriculum & Threshold set by intuition & May exceed what any policy achieves & Require $\tau\le\E_{\pi_{\mathrm{ref}}}[m]$ from an open-loop rollout \\
Curriculum & Progress judged by level reached & A gate that never refuses carries no information & Log and inspect the refusal count \\
Curriculum & Level stored in the environment & Fine-tuning silently changes the objective & Checkpoint $(\theta,\text{optimiser},k)$; compare on fixed $d_{\mathrm{eval}}$ \\
\midrule
Measurement & Outcomes from stochastic rollouts & Noise manufactures thresholded outcomes & Deterministic evaluation; contact-verified predicates \\
Measurement & Synchronised batched disturbance & Phase aliasing; bias independent of $N$ & Per-robot disturbance time over $W\ge P$ \\
Measurement & Aggregate by convenience & Existential computed for a universal claim & Match the quantifier to the claim \\
\midrule
Reference & Track the reference; reshape reward on failure & Infeasible reference: optimum is a projection onto $\mathcal F$ with different contacts & Static feasibility test, kinematic repair, open-loop check; residual feed-forward \\
\midrule
Growth & Retrain when observations change & Discards learned skills & Zero-padded input widening (function preserving) \\
Growth & Freeze non-policy joints in training & Point-mass training distribution & Observed exogenous motion with hidden command \\
\bottomrule
\end{tabular}
\end{table*}

%=======================================================================
\section{Discussion}
%=======================================================================

\textbf{Why the reward is not special.} Under the traditional view the reward is the
only optimised quantity, and so the only one that can be hacked. Of the divergences
above, the curriculum dilution and the noise-manufactured outcome were selected for by
optimisation even though neither quantity is a reward; phase aliasing and the quantifier
error needed no optimisation at all; and an infeasible reference is a specification
error that no optimiser can repair. What unites them is that each reported success
through a channel the engineer had already decided to trust. Validation effort should
follow trust, not only optimisation pressure.

\textbf{Cost.} None of the reformulations requires additional training. They cost one
open-loop rollout per reference and curriculum level, one deterministic replay per
reported outcome, a randomised timing schedule in the evaluation harness, and a few
lines in the checkpoint format.

\textbf{Scope.} All results are simulated, from a single training seed. The robot's
mass model derives from solid-material density, and actuator limits are those of the
simulation rather than the intended motors.

%=======================================================================
\section{Conclusion}
%=======================================================================

The traditional RL pipeline for legged robots treats the reward as the one place where a
proxy can diverge from its target. I have argued, with formal conditions and measured
instances, that curriculum gates, evaluation statistics and reference motions are proxies
of the same kind, and that each has a precise reformulation: first-order costs at rest
states, peak and outcome-grounded gates that are reachable and informative, curriculum
state saved with the model, deterministic and phase-desynchronised evaluation, and
feasibility-first references with residual feed-forward. Combined with function-preserving
input widening and split authority, these let a single humanoid policy grow across stages
without retraining, while keeping every number that steers it honest.

%=======================================================================
\section*{Reproducibility}
%=======================================================================

All figures are regenerated from the original training logs by the scripts accompanying
this manuscript. Environment configurations and reward terms are available from the
author on request.

\section*{Acknowledgment}
This work was carried out independently, without institutional affiliation or external
funding. It builds on NVIDIA Isaac Sim, IsaacLab and the \texttt{rsl\_rl}
implementation of PPO.

%=======================================================================


\begin{thebibliography}{00}

\bibitem{rudin2021}
N.~Rudin, D.~Hoeller, P.~Reist, and M.~Hutter, ``Learning to walk in minutes using
massively parallel deep reinforcement learning,'' in \emph{Proc. Conf. Robot Learning
(CoRL)}, 2021.

\bibitem{mittal2025}
M.~Mittal, P.~Roth \emph{et al.}, ``Isaac Lab: A GPU-accelerated simulation framework
for multi-modal robot learning,'' \emph{arXiv:2511.04831}, 2025.

\bibitem{peng2018}
X.~B.~Peng, P.~Abbeel, S.~Levine, and M.~van~de~Panne, ``DeepMimic: Example-guided deep
reinforcement learning of physics-based character skills,'' \emph{ACM Trans. Graphics},
vol.~37, no.~4, 2018.

\bibitem{skalse2022}
J.~Skalse, N.~Howe, D.~Krasheninnikov, and D.~Krueger, ``Defining and characterizing
reward hacking,'' in \emph{Advances in Neural Information Processing Systems
(NeurIPS)}, 2022.

\bibitem{pan2022}
A.~Pan, K.~Bhatia, and J.~Steinhardt, ``The effects of reward misspecification: Mapping
and mitigating misaligned models,'' in \emph{Proc. Int. Conf. Learning Representations
(ICLR)}, 2022.

\bibitem{gao2023}
L.~Gao, J.~Schulman, and J.~Hilton, ``Scaling laws for reward model overoptimization,''
in \emph{Proc. Int. Conf. Machine Learning (ICML)}, PMLR vol.~202, pp.~10835--10866,
2023.

\bibitem{laidlaw2025}
C.~Laidlaw, S.~Singhal, and A.~D.~Dragan, ``Correlated proxies: A new definition and
improved mitigation for reward hacking,'' in \emph{Proc. Int. Conf. Learning
Representations (ICLR)}, 2025.

\bibitem{tang2025}
C.~Tang, B.~Abbatematteo, J.~Hu, R.~Chandra, R.~Mart\'{i}n-Mart\'{i}n, and P.~Stone,
``Deep reinforcement learning for robotics: A survey of real-world successes,''
\emph{Annu. Rev. Control, Robot., Auton. Syst.}, vol.~8, pp.~153--188, 2025.

\bibitem{radosavovic2024}
I.~Radosavovic, T.~Xiao, B.~Zhang, T.~Darrell, J.~Malik, and K.~Sreenath, ``Real-world
humanoid locomotion with reinforcement learning,'' \emph{Science Robotics}, vol.~9,
no.~89, 2024.

\bibitem{he2025}
T.~He \emph{et al.}, ``ASAP: Aligning simulation and real-world physics for learning agile
humanoid whole-body skills,'' in \emph{Proc. Robotics: Science and Systems (RSS)}, 2025.

\bibitem{liao2025}
Q.~Liao, T.~E.~Truong, X.~Huang, Y.~Gao, G.~Tevet, K.~Sreenath, and C.~K.~Liu,
``BeyondMimic: From motion tracking to versatile humanoid control via guided diffusion,''
\emph{arXiv:2508.08241}, 2025.

\bibitem{ma2024}
Y.~J.~Ma, W.~Liang, G.~Wang, D.-A.~Huang, O.~Bastani, D.~Jayaraman, Y.~Zhu, L.~Fan, and
A.~Anandkumar, ``Eureka: Human-level reward design via coding large language models,''
in \emph{Proc. Int. Conf. Learning Representations (ICLR)}, 2024.

\bibitem{margolis2024}
G.~B.~Margolis, G.~Yang, K.~Paigwar, T.~Chen, and P.~Agrawal, ``Rapid locomotion via
reinforcement learning,'' \emph{Int. J. Robotics Research}, vol.~43, no.~4,
pp.~572--587, 2024.

\bibitem{agarwal2021}
R.~Agarwal, M.~Schwarzer, P.~S.~Castro, A.~Courville, and M.~G.~Bellemare, ``Deep
reinforcement learning at the edge of the statistical precipice,'' in \emph{Advances in
Neural Information Processing Systems (NeurIPS)}, 2021.

\bibitem{patterson2024}
A.~Patterson, S.~Neumann, M.~White, and A.~White, ``Empirical design in reinforcement
learning,'' \emph{J. Machine Learning Research}, vol.~25, pp.~1--63, 2024.

\bibitem{peng2021}
X.~B.~Peng, Z.~Ma, P.~Abbeel, S.~Levine, and A.~Kanazawa, ``AMP: Adversarial motion
priors for stylized physics-based character control,'' \emph{ACM Trans. Graphics},
vol.~40, no.~4, 2021.

\bibitem{johannink2019}
T.~Johannink \emph{et al.}, ``Residual reinforcement learning for robot control,'' in
\emph{Proc. IEEE Int. Conf. Robotics and Automation (ICRA)}, 2019.

\bibitem{chen2016}
T.~Chen, I.~Goodfellow, and J.~Shlens, ``Net2Net: Accelerating learning via knowledge
transfer,'' in \emph{Proc. Int. Conf. Learning Representations (ICLR)}, 2016.

\bibitem{nikishin2023}
E.~Nikishin, J.~Oh, G.~Ostrovski, C.~Lyle, R.~Pascanu, W.~Dabney, and A.~Barreto, ``Deep
reinforcement learning with plasticity injection,'' in \emph{Advances in Neural
Information Processing Systems (NeurIPS)}, 2023.

\bibitem{liu2025}
J.~Liu, J.~Obando-Ceron, A.~Courville, and L.~Pan, ``Neuroplastic expansion in deep
reinforcement learning,'' in \emph{Proc. Int. Conf. Learning Representations (ICLR)},
2025.

\bibitem{dohare2024}
S.~Dohare, J.~F.~Hernandez-Garcia, Q.~Lan, P.~Rahman, A.~R.~Mahmood, and R.~S.~Sutton,
``Loss of plasticity in deep continual learning,'' \emph{Nature}, vol.~632,
pp.~768--774, 2024.

\bibitem{juliani2024}
A.~Juliani and J.~T.~Ash, ``A study of plasticity loss in on-policy deep reinforcement
learning,'' in \emph{Advances in Neural Information Processing Systems (NeurIPS)}, 2024.

\bibitem{schulman2017}
J.~Schulman, F.~Wolski, P.~Dhariwal, A.~Radford, and O.~Klimov, ``Proximal policy
optimization algorithms,'' \emph{arXiv:1707.06347}, 2017.

\bibitem{schulman2015}
J.~Schulman, P.~Moritz, S.~Levine, M.~Jordan, and P.~Abbeel, ``High-dimensional
continuous control using generalized advantage estimation,'' \emph{arXiv:1506.02438},
2015.

\end{thebibliography}
\end{document}